\documentclass[12pt,final]{l4dc2021} 
\usepackage{times}
\usepackage{natbib}
\hypersetup{
    colorlinks,
    breaklinks,
    linkcolor = blue,
    citecolor = blue,
    urlcolor  = blue,
}
\usepackage{url}            
\usepackage{paralist}
\usepackage{booktabs}       
\usepackage{amsfonts}       
\usepackage{nicefrac}       
\usepackage{microtype}      
\usepackage{graphicx,color} 
\usepackage{enumitem}
\usepackage{algorithmic,algorithm,bm}
\usepackage{lipsum}
\usepackage{amsmath,mathrsfs,amssymb,hyperref}
\usepackage{caption2}
\usepackage{multirow}
\usepackage{booktabs}       
\usepackage{standalone}     
\usepackage{preview}
\usepackage{appendix}
\usepackage{xcolor}
\allowdisplaybreaks
\def \x {\mathbf{x}}
\def \y {\mathbf{y}}
\def \z {\mathbf{z}}
\def \u {\mathbf{u}}
\def \w {\mathbf{w}}

\def \v {\mathbf{v}}

\def \R {\mathbb{R}}
\def \T {\mathrm{T}}
\def \O  {\mathcal{O}}
\def \P {\mathcal{P}}
\def \S {\mathcal{S}}
\def \V {\mathcal{V}}
\def \X {\mathcal{X}}

\usepackage{amsthm}
\newtheorem{myThm}{Theorem}
\newtheorem{myLemma}{Lemma}
\theoremstyle{definition}
\newtheorem{myAssum}{Assumption}
\newtheorem{myInstance}{Instance}
\newtheorem{myRemark}{Remark}
\usepackage{mathtools}
\let\norm\undefined 
\DeclarePairedDelimiter\norm{\lVert}{\rVert}
\DeclarePairedDelimiter\abs{\lvert}{\rvert}
\newcommand\inner[2]{\langle #1, #2 \rangle}

\usepackage{amsthm}

\newenvironment{proof}{\par\noindent{\textbf{Proof}\ }}{\hfill\BlackBox\\[2mm]}

\DeclareMathOperator*{\argmin}{arg\,min}

\usepackage{algorithm,algorithmic}

\title[Improved Dynamic Regret of Strongly Convex and Smooth Functions]{Improved Analysis for Dynamic Regret of \\Strongly Convex and Smooth Functions}

\author{%
 \Name{Peng Zhao} \Email{zhaop@lamda.nju.edu.cn}\\
 \Name{Lijun Zhang} \Email{zhanglj@lamda.nju.edu.cn}\\
 \addr National Key Laboratory for Novel Software Technology\\
       Nanjing University, Nanjing 210023, China%
}
\begin{document}

\maketitle

\begin{abstract}
In this paper, we present an improved analysis for dynamic regret of strongly convex and smooth functions. Specifically, we investigate the Online Multiple Gradient Descent (OMGD) algorithm proposed by~\citet{NIPS'17:zhang-dynamic-sc-smooth}. The original analysis shows that the dynamic regret of OMGD is at most $\O(\min\{\P_T,\S_T\})$, where $\P_T$ and $\S_T$ are path-length and squared path-length that measures the cumulative movement of minimizers of the online functions. We demonstrate that by an improved analysis, the dynamic regret of OMGD can be improved to $\O(\min\{\P_T,\S_T,\V_T\})$, where $\V_T$ is the function variation of the online functions. Note that the quantities of $\P_T, \S_T, \V_T$ essentially reflect different aspects of environmental non-stationarity---they are not comparable in general and are favored in different scenarios. Therefore, the dynamic regret presented in this paper actually achieves a \emph{best-of-three-worlds} guarantee and is strictly tighter than previous results.
\end{abstract}


\section{Introduction}

In the development of online convex optimization, there are plenty of works devoted to designing online algorithms for minimizing static regret~\citep{book'16:Hazan-OCO}, defined as
\begin{equation}
	\mbox{S-Regret}_T = \sum_{t=1}^{T} f_t(\x_t) - \min_{\x \in \X} \sum_{t=1}^{T} f_t(\x),
\end{equation}
which is the difference between the cumulative loss of the online algorithm and that of the best strategy in hindsight. When  environment variables are stationary, minimizing static regret will lead to an algorithm that behaves well over the iterations. However, such a claim may not hold when environments are non-stationary and changing with time. To cope with non-stationary environments where the optimal decisions of online functions can be drifting over time, a more stringent measure---\emph{dynamic regret}---is proposed and draws much attentions in recent years~\citep{ICML'03:zinkvich,ICML'13:dynamic-model,OR'15:dynamic-function-VT,AISTATS'15:dynamic-optimistic,CDC'16:dynamic-sc,ICML'16:Yang-smooth,NIPS'17:zhang-dynamic-sc-smooth,ICML'18:zhang-dynamic-adaptive,NIPS'19:Wangyuxiang,AISTATS'20:Zhang,UAI'20:simple}, defined as
\begin{equation}
	\label{eq:dynamic-regret}
      \mbox{D-Regret}_T = \sum_{t=1}^T f_t(\x_t) -  \sum_{t=1}^T f_t(\x^*_t),
\end{equation}
where $\x_t^* \in \argmin_{\x\in \X} f_t(\x)$ is the minimizer of the online function. Dynamic regret enforces the player to compete with a time-varying comparator sequence, and thus is particularly favored in online learning in non-stationary environments. The notion of dynamic regret is also referred to as tracking regret or shifting regret in the setting of prediction with expert advice~\citep{journals/ml/HerbsterW98,JMLR'01:Herbster,JMLR'12:bousquet-dynamic,NIPS'16:Wei-non-stationary-expert,NIPS'19:Zheng}.
	
It is known that in the worst case, a sub-linear dynamic regret is not attainable unless imposing certain regularities of the comparator sequence or the function sequence~\citep{OR'15:dynamic-function-VT,AISTATS'15:dynamic-optimistic}. There are mainly three kinds of  regularities used in the literature~\citep{ICML'03:zinkvich,OR'15:dynamic-function-VT,NIPS'17:zhang-dynamic-sc-smooth}.
\begin{itemize}
	\item Path-length~\citep{ICML'03:zinkvich}: the variation of optimizers
	\begin{equation*}
		\label{eq:path-length}
		\P_T = \sum_{t=2}^{T} \norm{\x_{t-1}^* - \x_{t}^*}_2,
	\end{equation*}
	\item Squared path-length~\citep{NIPS'17:zhang-dynamic-sc-smooth}: the squared variation of optimizers
	\begin{equation*}
		\label{eq:square-path-length}
		\S_T = \sum_{t=2}^{T} \norm{\x_{t-1}^* - \x_{t}^*}_2^2,
	\end{equation*}
	\item Function variation~\citep{OR'15:dynamic-function-VT}: the variation over consecutive function values
	\begin{equation*}
		\label{eq:function-variation}
		\V_T = \sum_{t=2}^{T} \sup_{\x \in \X} \abs{f_{t-1}(\x) - f_t(\x)},
	\end{equation*}
\end{itemize}

When the path-length $\P_T$ is known in advance, dynamic regret of Online Gradient Descent~(OGD) is at most $\O(\sqrt{T(1 + \P_T)})$~\citep{ICML'03:zinkvich,ICML'16:Yang-smooth} for convex functions. For strongly convex and smooth functions,~\citet{CDC'16:dynamic-sc} first show that an $\O(\P_T)$ dynamic regret is achievable; later,~\citet{NIPS'17:zhang-dynamic-sc-smooth} propose Online Multiple Gradient Descent (OMGD) and prove an $\O(\min\{\P_T,\S_T\})$ dynamic regret. \citet{ICML'16:Yang-smooth} disclose that the $\O(\P_T)$ rate is also attainable for convex and smooth functions, provided that all the minimizers $\x_t^*$'s lie in the interior of the feasible set $\X$. Besides,~\citet{OR'15:dynamic-function-VT} show that OGD with a restarting strategy attains an $\O(T^{2/3}\V_T^{1/3})$ dynamic regret when the function variation $\V_T$ is available ahead of time. Later,~\citet{NIPS'19:Wangyuxiang} improve the rate to $\O(T^{1/3}\V_T^{2/3})$ for $1$-dimensional square loss by trend filtering techniques. We finally remark that another strengthened form of dynamic regret is recently studied~\citep{NIPS'18:Zhang-Ader,AISTATS'20:BCO,NIPS'20:sword} which supports to compete with any sequence of changing comparators rather than the optimizers of online functions only.

In this paper, we focus on the dynamic regret measure~\eqref{eq:dynamic-regret} of online convex optimization for \emph{strongly convex} and \emph{smooth} functions. Specifically, we assume that the online functions $f_1,\ldots,f_T$ are $\lambda$-strongly convex and $L$-smooth, namely, for all $t = 1,\ldots,T$,
\begin{itemize}
	\item $\lambda$-strong convexity: for any $\x, \y \in \X$, the following condition holds
	\begin{equation*}
	  f_t(\y) \geq f_t(\x) + \nabla f_t(\x)^\mathrm{T}(\y - \x) + \frac{\lambda}{2}\norm{\y - \x}_2^2.
	\end{equation*}
	\item $L$-smoothness: for any $\x, \y \in \X$, the following condition holds
	\begin{equation*}
	  f_t(\y) \leq f_t(\x) + \nabla f_t(\x)^\mathrm{T}(\y - \x) + \frac{L}{2}\norm{\y - \x}_2^2.
	\end{equation*}
\end{itemize}

To minimize the dynamic regret of strongly convex and smooth functions, \citet{NIPS'17:zhang-dynamic-sc-smooth} propose an algorithm called Online Multiple Gradient Descent (OMGD) and prove that under certain mild assumptions, OMGD enjoys the following dynamic regret
\begin{equation}
	\label{eq:dynamic-regret-OMGD-Zhang}
	\sum_{t=1}^{T} f_t(\x_t) - \sum_{t=1}^{T} f_t(\x_t^*) \leq  \O(\min\{\P_T, \mathcal{S}_T\}).
\end{equation}

In this paper, we present an improved analysis and demonstrate that the dynamic regret of OMGD can be also bounded by the function variation term $\V_T$. As a result, we actually show that the dynamic regret of OMGD is bounded by the minimization of path-length, squared path-length and function variation, namely,
\begin{equation}
	\label{eq:dynamic-regret-OMGD}
	\sum_{t=1}^{T} f_t(\x_t) - \sum_{t=1}^{T} f_t(\x_t^*) \leq  \O(\min\{\P_T, \S_T, \V_T\}).
\end{equation}

As shown in the work of~\citet{AISTATS'15:dynamic-optimistic}, these three kinds of regularities $\P_T, \S_T, \V_T$ are generally incomparable and are favored in different scenarios. Therefore, the dynamic regret bound~\eqref{eq:dynamic-regret-OMGD} presented in this paper actually achieves a \emph{best-of-three-worlds} guarantee, and is strictly tighter than the bound~\eqref{eq:dynamic-regret-OMGD-Zhang} of~\citet{NIPS'17:zhang-dynamic-sc-smooth}.

\section{Algorithm: Online Multiple Gradient Descent}
\label{sec:algorithm}
We first introduce the Online Gradient Descent (OGD) algorithm~\citep{ICML'03:zinkvich}, which starts from any $\x_1 \in \X$ and performs the following update at each iteration:
\begin{equation*}
	\x_{t+1} = \Pi_{\X}[\x_t - \eta \nabla f_t(\x_t)],
\end{equation*}
where $\eta > 0$ is the step size, and $\Pi_{\X}[\cdot]$ denotes  Euclidean projection onto the nearest point in $\X$. 

To minimize the dynamic regret of strongly convex and smooth functions,~\citet{NIPS'17:zhang-dynamic-sc-smooth} propose a variant of OGD, called Online Multiple Gradient Descent (OMGD). The algorithm performs gradient descent multiple times at each iteration. Specifically, at iteration $t$, given the current decision $\x_t$, OMGD will produce a sequence of $\z_t^{1},\ldots,\z_t^{K},\z_t^{K+1}$, where $K$ is the number of inner iterations, a constant \emph{independent} from the time horizon $T$. The inner iterations start from $\z_t^1 = \x_t$ and then perform the following update procedure,
\begin{equation*}
	\z_t^{k+1} = \Pi_{\X}[\z_t^{k} - \eta \nabla f_t(\z_t^{k})],
\end{equation*}
where $k = 1,\ldots,K$ is the index of inner loop. The decision $\x_{t+1}$ is set as the output of the inner iterations $\z_{t}^{K+1}$, i.e., $\x_{t+1} = \z_{t}^{K+1}$. The procedures of OMGD are summarized in Algorithm~\ref{alg:OMGD}.

\begin{algorithm}[t]
\caption{Online Multiple Gradient Descent (OMGD)~\citep{NIPS'17:zhang-dynamic-sc-smooth}}
\begin{algorithmic}[1]
\REQUIRE number of inner iterations $K$ and step size $\eta$
\STATE Let $\x_1$ be any point in $\X$
\FOR{$t=1,\ldots,T$}
\STATE Submit $\x_t$ and receive the loss $f_t:\X \mapsto \R$
\STATE $\z_t^1= \x_t$
\FOR{$k = 1,\ldots,K$}
\STATE \[
\z_t^{k+1} = \Pi_{\X}[\z_t^k - \eta\nabla f_t(\z_t^k)]
\]
\ENDFOR
\STATE $\x_{t+1} = \z_t^{K+1}$
\ENDFOR
\end{algorithmic}
\label{alg:OMGD}
\end{algorithm}

\section{Dynamic Regret Analysis}
In this section, we provide dynamic regret analysis for the OMGD algorithm. We first restate the (squared) path-length bounds of~\citet{NIPS'17:zhang-dynamic-sc-smooth} in Section~\ref{sec:path-length-bound} and then prove the function variation bounds in Section~\ref{sec:function-variation-bound}. Finally, in Section~\ref{sec:comparison} we present comparisons between various regularities and show the advantage of our results.

Before presenting the theoretical analysis, we state the following standard assumption adopted in the work of~\citet{NIPS'17:zhang-dynamic-sc-smooth} and this work. 
\begin{myAssum}
\label{assm:function}
Suppose the following conditions hold for each online function $f_t: \X \mapsto \R$:
\begin{itemize}
\item The online function $f_t$ is $\lambda$-strongly convex and $L$-smooth over $\X$;
\item The gradients are bounded by $G$, i.e., $\norm{\nabla f_t(\x)}_2 \leq G$ for any $\x \in \X$.
\end{itemize}
\end{myAssum}

\subsection{(Squared) Path-length Bounds}
\label{sec:path-length-bound}
\citet{NIPS'17:zhang-dynamic-sc-smooth} prove that the dynamic regret of OMGD can be bounded by the path-length and squared path-length. We restate their results as follows.
\begin{myThm}[{Theorem 2 of~\citet{NIPS'17:zhang-dynamic-sc-smooth}}]
\label{thm:path-length-bound}
Under Assumption~\ref{assm:function}, by setting the step size $\eta \leq 1/L$ and the number of inner iterations $K=\lceil\frac{1 / \eta+\lambda}{2 \lambda} \ln 4\rceil$ in Algorithm~\ref{alg:OMGD}, for any constant $\alpha > 0$, we have
\begin{equation}
\label{eq:thm-P_T}
\begin{split}
&\sum_{t=1}^T  f_t(\x_t) - \sum_{t=1}^{T} f_t(\x_t^*)
\leq \min\left\{
\begin{split}
&  2 G \P_T +  2 G \norm{\x_1 - \x_1^*}_2,\\
&  \frac{1}{2\alpha}\sum_{t=1}^T  \norm{\nabla f_t(\x_{t}^*)}_2^2  + 2 (L+\alpha) \mathcal{S}_T + (L+\alpha)\norm{\x_{1} - \x_{1}^*}_2^2.
\end{split} \right.
\end{split}
\end{equation}

Furthermore, suppose $\sum_{t=1}^{T}\norm{\nabla f_t(\x_t^*)}_2^2 = \O(\S_T)$, we have 
\begin{equation}
	\label{eq:path-length-bound}
	\sum_{t=1}^{T} f_t(\x_t) - \sum_{t=1}^{T} f_t(\x_t^*) \leq \O\big(\min\{\P_T, \S_T\}\big).
\end{equation}
In particular, if $\x_t^*$ belongs to the relative interior of $\X$ (i.e., $\nabla f_t(\x_t^*) = 0$) for all $t \in [T]$, the dynamic regret bound in~\eqref{eq:thm-P_T}, as $\alpha \rightarrow 0$, implies
\[
\begin{split}
&\sum_{t=1}^T  f_{t}(\x_t) - \sum_{t=1}^{T} f_{t}(\x_t^*)
\leq  \min\big\{2 G \P_T +  2 G \norm{\x_1 - \x_1^*}_2, 2 L \S_T + L\norm{\x_{1} - \x_{1}^*}_2^2\big\}.
\end{split}
\]
\end{myThm}

\subsection{Function Variation Bounds}
\label{sec:function-variation-bound}
In this part, we show that by an improved analysis, the dynamic regret of OMGD can be further bounded by the function variation $\V_T$, as demonstrated in the following theorem.
\begin{myThm}
\label{thm:VT-bound}
Under Assumption~\ref{assm:function}, by setting the step size $\eta = 1/L$, and the number of inner iterations $K = \lceil\frac{4(L + \lambda)}{\lambda}\ln 4\rceil$ in Algorithm~\ref{alg:OMGD}, we have
\begin{equation}
	\label{eq:dynamic-regret-V_t}
	\sum_{t=1}^{T} f_t(\x_t) - \sum_{t=1}^{T} f_t(\x_t^*) \leq 2 \V_T + 2\big( f_{1}(\x_1)-f_T(\x_{T+1})\big).
\end{equation}

Furthermore, suppose $\sum_{t=1}^{T}\norm{\nabla f_t(\x_t^*)}_2^2 = \O(\S_T)$, from Theorem~\ref{thm:path-length-bound} and~\eqref{eq:dynamic-regret-V_t}, we have 
\begin{equation*}
	\sum_{t=1}^{T} f_t(\x_t) - \sum_{t=1}^{T} f_t(\x_t^*) \leq \O\big(\min\{\P_T, \S_T, \V_T\}\big).
\end{equation*}
In particular, if $\x_t^*$ belongs to the relative interior of $\X$ (i.e., $\nabla f_t(\x_t^*) = 0$) for all $t \in [T]$, the dynamic regret bounds in~\eqref{eq:thm-P_T} and~\eqref{eq:dynamic-regret-V_t}, as $\alpha \rightarrow 0$, imply
\[
\begin{split}
\sum_{t=1}^T  f_{t}(\x_t) - \sum_{t=1}^{T} f_{t}(\x_t^*)
\leq  \min\big\{2 G \P_T +  2 G \norm{\x_1 - \x_1^*}_2, {}& 2 L \S_T + L\norm{\x_{1} - \x_{1}^*}_2^2, \\
{}& 2 \V_T + 2\big( f_{1}(\x_1)-f_T(\x_{T+1})\big)\big\}.
\end{split}
\]
\end{myThm}

\begin{myRemark}
Notice that above settings of step size $\eta = 1/L$ and the number of inner iterations $K = \lceil\frac{4(L + \lambda)}{\lambda}\ln 4\rceil$ also satisfy the condition of Theorem~\ref{thm:path-length-bound} (namely, $\eta \leq 1/L$ and $K \geq \lceil\frac{L+\lambda}{2 \lambda} \ln 4\rceil$ hold simultaneously), so the dynamic regret is also bounded by the path-length bounds in~\eqref{eq:path-length-bound}.
\end{myRemark}

To prove Theorem~\ref{thm:VT-bound}, we introduce the following key lemma due to~\citet{nesterov2013gradient}.

\begin{myLemma}
\label{lemma:func-contract}
Assume that the function $f: \X \mapsto \R$ is $\lambda$-strongly convex and $L$-smooth, and denote by $\x^*$ the optimizer, i.e., $\x^* = \argmin_{\x \in \X} f(\x)$. Let 
\begin{equation}
	\label{eq:update-procedure}
	\v = \Pi_{\X}\left[\u - \frac{1}{L}\nabla f(\u)\right].
\end{equation}
Then, we have
\begin{equation}
	\label{eq:func-contract}
	f(\v) - f(\x^*) \leq \gamma \big(f(\u) - f(\x^*)\big),
\end{equation}
where 
\begin{equation*}
	\gamma = 
	\begin{cases}
	\frac{1}{2}, &\text{if }~3 \lambda \geq 2 L \vspace{2mm}\\
	1-\frac{\lambda}{4(L-\lambda)}, &\text{otherwise.}
\end{cases}
\end{equation*}
\end{myLemma}

\begin{proof}[{of Lemma~\ref{lemma:func-contract}}]
Notice that the update procedure in~\eqref{eq:update-procedure} is equivalent to 
\begin{equation}
	\label{eq:update-procedure-2}
	\v = \argmin_{\x \in \X}\left\{ f(\u) + \inner{\nabla f(\u)}{\x - \u} + \frac{L}{2}\norm{\x - \u}_2^{2}\right\}.
\end{equation}

Then, we have
\begin{align*}
f(\v) \leq {} & f(\u)+\inner{\nabla f(\u)}{\v - \u} + \frac{L}{2}\norm{\v - \u}_2^2 \tag*{(by $L$-smoothness)}\\
=&\min _{\x \in \X}\left\{f(\u)+\left\langle\nabla f(\u),  \x-\u\right\rangle+\frac{L}{2}\norm{\x - \u}_2^2\right\}\tag*{(due to update procedure in~\eqref{eq:update-procedure-2})}\\
\leq {} & \min _{\x \in \X}\left\{f(\x)-\frac{\lambda}{2}\norm{\x - \u}_2^2+\frac{L}{2}\norm{\x - \u}_2^2\right) \tag*{(by $\lambda$-strong convexity)}\\
\leq {} & \min _{\x=\alpha \x^*+(1-\alpha) \u,  \alpha \in[0, 1]}\left\{f(\x)+\frac{ L - \lambda }{2}\norm{\x - \u}_2^2\right\}\\
= {} &\min _{\alpha \in[0, 1]}\left\{f\left(\alpha \x^*+(1-\alpha) \u\right)+\frac{L-\lambda}{2}\norm{\alpha \x^*+(1-\alpha) \u-\u}_2^2\right\}\\
\leq {} & \min _{\alpha \in[0, 1]}\left\{\alpha f(\x^*)+(1-\alpha) f(\u)+\frac{L-\lambda}{2} \alpha^2\norm{\x^*-\u}_2^2\right\}\\
= {} &\min _{\alpha \in[0, 1]}\left\{f(\u)-\alpha\left(f(\u)-f(\x^*)\right)+\frac{L-\lambda}{2} \alpha^2\norm{\x^*-\u}_2^2\right\}\\
\leq {} & \min _{\alpha \in[0, 1]}\left\{f(\u)-\alpha\left(f(\u)-f(\x^*)\right)+\frac{L-\lambda}{2} \frac{2}{\lambda} \alpha^2\left(f(\u)-f(\x^*)\right)\right\}\\
= {} & \min _{\alpha \in[0, 1]}\left\{f(\u)+\left(\frac{L-\lambda}{\lambda} \alpha^2-\alpha\right)\big(f(\u)-f(\x^*)\big)\right\}.
\end{align*}
The last inequality is true because $f(\x) - f(\x^*) \geq \frac{\lambda}{2} \norm{\x - \x^*}_2^2$ holds for all $\x \in \X$ due to the strong convexity~\citep[Theorem 2.1.8]{book'18:Nesterov-OPT}. Therefore, if $\frac{\lambda}{2(L - \lambda)} \geq 1$, then we set $\alpha = 1$ and obtain 
\[
	f(\v)-f(\x^*) \leq \frac{L-\lambda}{\lambda}\left(f(\u)-f(\x^*)\right) \leq \frac{1}{2}\left(f(\u)-f(\x^*)\right).
\]
Otherwise, we set $\alpha = \frac{\lambda}{2(L - \lambda)}$, and obtain
\[
	f(\v)-f(\x^*) \leq\left(1-\frac{\lambda}{4(L-\lambda)}\right)\left(f(\u)-f(\x^*)\right).
\]
This ends the proof of Lemma~\ref{lemma:func-contract}.
\end{proof}

Based on Lemma~\ref{lemma:func-contract}, we now present the proof of Theorem~\ref{thm:VT-bound}.
~\\

\begin{proof}[of Theorem~\ref{thm:VT-bound}]
From Lemma~\ref{lemma:func-contract}, we know that 
\begin{equation}
	\label{eq:func-contract-apply}
	f_t(\x_{t+1}) - f_t(\x_t^*) = f_t(\z_t^{K+1}) - f_t(\x_t^*) \overset{\eqref{eq:func-contract}}{\leq} \gamma^K \left(f_t(\x_t)-f_t(\x_t^*)\right) \leq \frac{1}{4}\left(f_t(\x_t)-f_t(\x_t^*)\right).
\end{equation}
The last inequality holds due to the following facts. From the setting of the inner iteration $K = \lceil\frac{4(L + \lambda)}{\lambda}\ln 4\rceil$, on one hand, it is clear that $\gamma^K \leq \frac{1}{4}$ holds when $\gamma = \frac{1}{2}$. On the other hand, when $\gamma = 1 - \frac{\lambda}{4(L-\lambda)}$, we have $\left(1 - \frac{\lambda}{4(L-\lambda)}\right)^K \leq \exp\left(-\frac{\lambda K}{4(L-\lambda)}\right) \leq \frac{1}{4}.$

Therefore, we can upper bound the dynamic regret as follows.
\begin{equation}
\label{eq:variation}
\begin{split}
	\sum_{t=1}^{T} f_t(\x_t) - f_t(\x_t^*) \leq {} & f_1(\x_1) - f_1(\x_1^*) + \sum_{t=2}^{T} f_t(\x_t) - f_{t-1}(\x_t) + f_{t-1}(\x_t) - f_t(\x_t^*) \\
	\leq {} & f_1(\x_1) - f_1(\x_1^*) + \V_T + \sum_{t=2}^{T} f_{t-1}(\x_t) - f_t(\x_t^*)\\
	= {} & f_1(\x_1) - f_T(\x_{T+1}) + \V_T + \sum_{t=1}^{T-1} f_{t}(\x_{t+1}) - f_t(\x_t^*) \\
	\overset{\eqref{eq:func-contract-apply}}{\leq} {} & f_1(\x_1) - f_T(\x_{T+1}) + \V_T + \frac{1}{4}\sum_{t=1}^{T-1} f_t(\x_t)-f_t(\x_t^*). 
\end{split}
\end{equation}

Thus, by rearranging above terms, we prove the statement in Theorem~\ref{thm:VT-bound}.
\end{proof}

\subsection{Comparisons of Path-length and Function Variation} 
\label{sec:comparison}
As demonstrated by~\citet{AISTATS'15:dynamic-optimistic}, the path-length and function variation are not comparable in general. Let us consider the following two instances.
\begin{myInstance}[{Online linear optimization over a $d$-dimensional simplex}] Consider the online linear optimization problem, the feasible set $\X$ is set as $\Delta_d = \{\x \mid \x \in \R^d, x_i \geq 0, \sum_{i=1}^d x_i=1 \}$ and the online functions are $f_t(\x) = \inner{\w_t}{\x}$, where 
\begin{equation*}
 	\w_t =
 	\begin{cases}
 	[\frac{1}{T},0,0,\ldots,0]^{\T} & \text{ when $t$ is odd}\\
 	[0,\frac{1}{T},0,\ldots,0]^{\T} & \text{ when $t$ is even}
 	\end{cases}
\end{equation*} 	
	Then, $\x_t^* = [1,0,0,\ldots,0]^{\T}$ when $t$ is odd and $\x_t^* = [0,1,0,\ldots,0]^{\T}$ when $t$ is even. So we have 
	\begin{equation*}
		\P_T = \S_T = \Theta(T),\quad \V_T = \Theta(1).
	\end{equation*}
\end{myInstance} 

\begin{myInstance}[{Online linear optimization over a $2$-dimensional simplex, or prediction with two-expert advice}] Let the feasible set $\X$ be as $\Delta_2 = \{\x \mid \x \in \R^2, x_i \geq 0, x_1 + x_2 = 1 \}$, and the online functions be $f_t(\x) = \inner{\w_t}{\x}$, where 
\begin{equation*}
 	\w_t =
 	\begin{cases}
 	[-\frac{1}{2},0]^{\T} & \text{ when $t$ is odd}\\
 	[0,\frac{1}{2}]^{\T} & \text{ when $t$ is even}
 	\end{cases}
\end{equation*} 
Then, we know that the optimal decision is fixed as the first expert, that is, $\x_t^* = [1,0]^{\T}$. Thus,
\begin{equation*}
	\P_T = \S_T = 0,\quad \V_T = \Theta(T).
\end{equation*}
\end{myInstance} 

From above two instances, we conclude that function variation and (squared) path-length are not comparable in general. Our analysis shows that OMGD actually enjoys an $\O(\min\{\P_T,\S_T,\V_T\})$ dynamic regret guarantee, which achieves \emph{the best of three worlds} and thus strictly improves the previous result of $\O(\min\{\P_T,\S_T\})$ by~\citet{NIPS'17:zhang-dynamic-sc-smooth}.

\section{Discussion}
\label{sec:discussion}
In this section, we discuss some aspects of our results.

\subsection{Relationship between Squared Path-length and Function Variation}
We further discuss the relationship between $\S_T$ and $\V_T$ providing that the online functions are $\lambda$-strongly convex. Denote by $\x_t^* = \argmin_{\x\in\X} f_t(\x)$ and $\x_{t-1}^* = \argmin_{\x\in\X} f_{t-1}(\x)$. Indeed,
\begin{align*}
	\norm{\x_t^* - \x_{t-1}^*}_2^2 \leq {} & \frac{1}{\lambda} \Big(f_{t-1}(\x_t^*) - f_{t-1}(\x_{t-1}^*)\Big)\\
	= {} & \frac{1}{\lambda} \Big(f_{t-1}(\x_t^*) - f_t(\x_t^*) + f_t(\x_t^*) - f_{t-1}(\x_{t-1}^*)\Big)\\
	\leq {} & \frac{1}{\lambda} \Big(f_{t-1}(\x_t^*) - f_t(\x_t^*) + f_t(\x_{t-1}^*) - f_{t-1}(\x_{t-1}^*)\Big)\\
	\leq {} & \frac{2}{\lambda} \sup_{\x \in \X} \abs{f_t(\x) - f_{t-1}(\x)}.
\end{align*}
Summing over all iterations yields $\S_T \leq 2\V_T/\lambda$. However, the right-hand side exhibits an explicit dependency on the strong convexity modulus $\lambda$, which could be very large when $\lambda$ is small. Our analysis shows that the undesirable dependency can be eliminated by multiple gradient descent per round, such that the dynamic regret can be upper bounded by $\V_T + 2(f_{1}(\x_1)-f_T(\x_{T+1}))$ without the $1/\lambda$ factor, as demonstrated in Theorem~\ref{thm:VT-bound}.

\subsection{Extensions to Non-strongly Convex Functions}
When the online functions are not strongly convex, we discover that the following \emph{greedy strategy} also enjoys nice dynamic regret guarantees. The greedy strategy picks one of the minimizers of the last online function as the current decision, namely,
\begin{equation}
	\label{eq:greedy-benchmark}
	\x_{t+1} = \argmin_{\x \in \X} f_t(\x).
\end{equation}
Since the online function $f_t$ is not guaranteed to be strongly convex, it may have multiple minimizers. Denote by $\X_t^*$ the set of all its minimizers, then the greedy strategy can choose an arbitrary one from $\X_t^*$ as the current decision. In essence, this greedy strategy can be regarded as a version of OMGD with a sufficiently large number of inner iterations.

The following theorem demonstrates the greedy strategy enjoys an $\O(\min\{\bar{\P}_T, \bar{\S}_T, \V_T\})$ dynamic regret. Note that the definitions of (squared) path-length terms are slightly different from previous ones to handle the potential non-uniqueness of minimizers. The proof is in Appendix~\ref{appendix:proof}.
\begin{myThm}
\label{thm:greedy-benchmark}
Under Assumption~\ref{assm:function} except for the strong convexity condition (i.e., it is allowed that $\lambda = 0$), suppose $\sum_{t=1}^{T}\norm{\nabla f_t(\x_t^*)}_2^2 = \O(\bar{\S}_T)$, then the greedy strategy~\eqref{eq:greedy-benchmark} satisfies
\begin{equation*}
	\sum_{t=1}^{T} f_t(\x_t) - \sum_{t=1}^{T} f_t(\x_t^*) \leq \O\big(\min\{\bar{\P}_T, \bar{\S}_T, \V_T\}\big),
\end{equation*}
where $\bar{\P}_T = \max_{\{\x_t^* \in \X_t^*\}_{t=1}^T}\sum_{t=2}^{T} \norm{\x_{t-1}^* - \x_{t}^*}_2$ is the path-length, $\bar{\S}_T =\max_{\{\x_t^* \in \X_t^*\}_{t=1}^T} \sum_{t=2}^{T} \norm{\x_{t-1}^* - \x_{t}^*}_2^2$ is the squared path-length, and $\V_T = \sum_{t=2}^{T} \sup_{\x \in \X} \abs{f_{t-1}(\x) - f_t(\x)}$ is the function variation. 
\end{myThm}
Note that the smoothness condition is only used in deriving the square path-length bound, which is not necessary for path-length bound and function variation bound.

\begin{myRemark}
We further elucidate the connection between the greedy strategy and OMGD. Both methods decide the current decision $\x_{t+1}$ based on the available online function $f_t$. The greedy strategy selects $\x_{t+1} = \x_t^* \in \argmin_{\x \in \X} f_t(\x)$ as the minimizer directly, while OMGD is a first-order method performing multiple gradient descent $\z_t^{k+1} = \Pi_{\X}[\z_t^k - \eta\nabla f_t(\z_t^k)]$ for $k=1,\ldots,K$ with $\z_t^1 = \x_t$ and $\x_{t+1}= \z_t^{K+1}$. Thus, OMGD can be regarded as a high-quality approximation of the greedy strategy via a \emph{finite} number of inner gradient descent, providing that online functions are strongly convex. We note that the initialization of $\z_t^1 = \x_t$ is necessary for achieving a \emph{constant} number of inner iterations (in our setting, $K = \lceil 4(L/\lambda + 1)\ln 4\rceil$ as shown in Theorem~\ref{thm:VT-bound}).
\end{myRemark}

\begin{myRemark}
It is worth noting that in the analysis of OMGD the strong convexity is only necessitated by the function-value decay lemma (Lemma~\ref{lemma:func-contract}) in our analysis. Thus, it would be possible to relax the strong convexity condition by some other properties that allow a similar function-value decay property exhibited in~\eqref{eq:func-contract}, with the purpose of accommodating broader class of problems.
\end{myRemark}

\section{Conclusion}
\label{sec:conclusion}
In this paper, we investigate an existing online algorithm (Online Multiple Gradient Descent, OMGD) proposed by~\citet{NIPS'17:zhang-dynamic-sc-smooth} for dynamic regret minimization of strongly convex and smooth functions. Under certain mild assumptions, OMGD was shown to attain $\O(\min\{\P_T, \mathcal{S}_T\})$~dynamic regret, where $\P_T$ and $\S_T$ are path-length and squared path-length, respectively. This paper contributes to an improved analysis and proves an $\O(\min\{\P_T, \S_T, \V_T\})$ dynamic regret without modifying the algorithm, where $\V_T$ is the function variation. The key technique used to realize the improvement is a careful usage of function-value decay lemma. Since different regularities $\P_T, \S_T, \V_T$ are generally not comparable and are favored in different scenarios, our presented dynamic regret achieves a best-of-three-worlds guarantee and is strictly tighter than previous results.

Many recent works consider control and dynamical systems in non-stationary environments~\citep{AISTATS'19:Goel,NIPS'19:Li,NIPS'20:Guanya,arXiv'21:Scream}. As a matter of fact, online learning plays an important role in modern computational control theory, especially the online non-stochastic control problem introduced by~\citet{ICML'19:online-control}. We believe that the dynamic regret minimization results developed in this paper could be of interest for further development of online control in non-stationary environments.

\section*{Acknowledgment}
This work was partially supported by the Open Research Projects of Zhejiang Lab (NO.~2021KB0AB02), and the Collaborative Innovation Center of Novel Software Technology and Industrialization. We are grateful for the anonymous reviewers for their valuable comments.

\appendix
\section{{Proof of Theorem~\ref{thm:greedy-benchmark}}}
\label{appendix:proof}
\begin{proof}[{of Theorem~\ref{thm:greedy-benchmark}}]
We first consider the path-length bound. Note that the path-length dynamic regret for this greedy strategy was firstly proved by~\citet{ICML'16:Yang-smooth}. Below we restate their proof.
\begin{align*}
\sum_{t=1}^T f_t(\x_t) - \sum_{t=1}^T f_t(\x_t^*) = {} & f_1(\x_1) - f_1(\x_1^*) + \sum_{t=2}^T f_t(\x_{t-1}^*) - \sum_{t=1}^T f_t(\x_t^*)\\
\leq {} & f_1(\x_1) - f_1(\x_1^*) + G \sum_{t=2}^T \norm{\x_{t-1}^* - \x_t^*}_2\\
\leq {} & f_1(\x_1) - f_1(\x_1^*) + G \bar{\P}_T = \O(\bar{\P}_T).
\end{align*}
Next, we prove the squared path-length bound. 
\begin{align*}
\sum_{t=1}^T f_t(\x_t) - \sum_{t=1}^T f_t(\x_t^*) = {} & f_1(\x_1) - f_1(\x_1^*) + \sum_{t=2}^T f_t(\x_{t-1}^*) - \sum_{t=1}^T f_t(\x_t^*) \\
\leq {} & f_1(\x_1) - f_1(\x_1^*) + \sum_{t=2}^T \left( \nabla f_t(\x_t^*)^\T(\x_{t-1}^* - \x_t^*) + \frac{L}{2} \norm{\x_{t-1}^* - \x_t^*}_2^2 \right)\\
\leq {} & f_1(\x_1) - f_1(\x_1^*) + \frac{1}{2} \sum_{t=2}^{T} \norm{\nabla f_t(\x_t^*)}_2^2 + \frac{L+1}{2} \sum_{t=2}^{T} \norm{\x_{t-1}^* - \x_t^*}_2^2\\
\leq {} & \O(\bar{\S}_T),
\end{align*}
where the first inequality exploits $L$-smoothness condition and the last inequality makes use of the assumption that $\sum_{t=1}^{T}\norm{\nabla f_t(\x_t^*)}_2^2 = \O(\bar{\S}_T)$.

We finally prove the function variation bound. Note that the function variation dynamic regret for this greedy strategy was firstly mentioned by~\citet{AISTATS'15:dynamic-optimistic} in the paragraph under Eq.~(5) of their paper. Following the derivation of~\eqref{eq:variation}, we have
\begin{align*}
\sum_{t=1}^T f_t(\x_t) - \sum_{t=1}^T f_t(\x_t^*) = {} & f_1(\x_1) - f_1(\x_1^*) + \sum_{t=2}^T f_t(\x_{t-1}^*) - \sum_{t=1}^T f_t(\x_t^*)\\
\leq & f_1(\x_1) - f_T(\x_T^*) + \sum_{t=2}^T \left(f_t(\x_{t-1}^*) - f_{t-1}(\x_{t-1}^*)\right)\\
\leq {} & f_1(\x_1) - f_1(\x_1^*) + \sum_{t=2}^T \sup_{\x \in \X} \abs{f_t(\x) - f_{t-1}(\x)} = \O(\V_T).
\end{align*}

Combining above three upper bounds of dynamic regret ends the proof.
\end{proof}
\newpage
\bibliography{online_learning}
\bibliographystyle{abbrvnat}
\end{document}